\documentclass{IEEEtran}

\usepackage{ifpdf}
\usepackage{cite}

\usepackage{amsmath}
\usepackage{algorithm}% http://ctan.org/pkg/algorithms
\usepackage{algpseudocode}% http://ctan.org/pkg/algorithmicx
\usepackage{array}
\ifCLASSOPTIONcompsoc
  \usepackage[caption=false,font=normalsize,labelfont=sf,textfont=sf]{subfig}
\else
  \usepackage[caption=false,font=footnotesize]{subfig}
\fi

\usepackage{fixltx2e}
\usepackage{stfloats}

\usepackage{url}

\usepackage{microtype}
\usepackage{graphicx}
\usepackage{booktabs} % for professional tables
\usepackage{hyperref}       % hyperlinks
\usepackage{tikz}
\usepackage{url}            % simple URL typesetting
\usepackage{amsfonts}       % blackboard math symbols
\usepackage{nicefrac}       % compact symbols for 1/2, etc.
\usepackage{amsmath}
\usepackage{amsthm}
\usepackage{multirow}
\usepackage{colortbl}
\usepackage{color}
\usepackage{array}

\newtheorem{thm}{Theorem}

\newtheorem{prop}{Proposition}
\newtheorem{cor}{Corollary}

\newtheorem{defn}{Definition}

\begin{document}
\title{rTop-$k$: A Statistical Estimation Approach to Distributed SGD}

\author{Leighton~Pate~Barnes,
        Huseyin~A.~Inan,
        Berivan~Isik,
        and~Ayfer~\"{O}zg\"{u}r
        \thanks{All authors contributed equally to this work.}
        \thanks{This work was supported in part by NSF award CCF-1704624 and by a Google faculty research award.}
        \thanks{L. P. Barnes, H. A. Inan, B. Isik and A. \"{O}zg\"{u}r are with Department of Electrical Engineering, Stanford University, Stanford, CA 94305, USA (e-mail: lpb@stanford.edu; hinan1@stanford.edu; berivan.isik@stanford.edu; aozgur@stanford.edu).  }
}

\maketitle

\IEEEpeerreviewmaketitle

\begin{abstract}
%Distributing large datasets and the optimization of model parameters over multiple nodes can speed up the training of large neural networks by taking advantage of data parallelism. It can also provide increased privacy and personalization in settings such as federated learning, where it is desirable to process the data of each user locally on its mobile device. However, 
The large communication cost for exchanging gradients between different nodes significantly limits the scalability of distributed training for large-scale learning models. Motivated by this observation, there has been significant recent interest in techniques that reduce the communication cost of distributed Stochastic Gradient Descent (SGD), with gradient sparsification techniques such as top-$k$ and random-$k$ shown to be particularly effective. The same observation has also motivated a separate line of work in distributed statistical estimation theory focusing on the impact of communication constraints on the estimation efficiency of different statistical models. The primary goal of this paper is to connect these two research lines and demonstrate how statistical estimation models and their analysis can lead to new insights in the design of communication-efficient training techniques. We propose a simple statistical estimation model for the stochastic gradients which captures the sparsity and skewness of their distribution. The statistically optimal communication scheme arising from the analysis of this model leads to a new sparsification technique for SGD, which concatenates random-$k$ and top-$k$, considered separately in the prior literature. We show through extensive experiments on both image and language domains with CIFAR-10, ImageNet, and Penn Treebank datasets that the concatenated application of these two sparsification methods consistently and significantly outperforms either method applied alone.
\end{abstract}

%\begin{IEEEkeywords}
%federated learning, minimax bounds, distributed estimation
%\end{IEEEkeywords}

\section{Introduction}
Stochastic Gradient Descent (SGD) and its variants have become the workhorse for training large machine learning models on ever growing datasets. Such large-scale training can be accelerated by partitioning datasets across multiple nodes and parallelizing the computation of the stochastic gradient; nodes can compute gradients in parallel based on their local datasets, and these local gradients can then be aggregated at a master node. However, for large models with millions of parameters, it is now well-understood that full-precision gradient aggregation is costly and the associated communication overhead can negate the savings in computation time  \cite{deepgrad}. %For example, \cite{Linetal} shows that distributing training  of AlexNet with 64 nodes leads to less than 7$\times$ speed-up compared with single-node training when nodes are connected with 1 Gbps Ethernet. However, increasing the network bandwidth to 10 Gbps in the same setting leads more than 35$\times$ speed-up in training. 
Communication is an even more significant bottleneck when training is performed over potentially slow, unreliable and expensive wireless links such as in federated learning \cite{fedreview}, where locally processing user data offers additional benefits such as privacy and personalization.  

Motivated by these observations, there has been significant recent interest in developing communication-efficient SGD methods \cite{Li, Seide, Strom, QSGD, federated0, federated1, federated2, Aji, TernGrad, signSGD, deepgrad,Wang, Wangni,vqSGD}, as well as other fundamental contributions to distributed first-order \cite{Bottou10large-scalemachine,NIPS2011_218a0aef,1104412,Qgadmm,MATCHA} and second-order \cite{Shamir,jahani2020efficient,ba2016distributed,crane2019dingo,jahani2019scaling} optimization procedures. These works show that the communication cost of SGD can be significantly reduced by using a variety of techniques, including thresholding and sparsification of the local gradients (e.g. communicating only the top-$k$ gradients with largest magnitudes or $k$ randomly selected gradients); quantization and compression of gradients (potentially with randomization) to a small number of bits;  and reducing the number of communication rounds by performing multiple iterations on the local dataset of each node (e.g. federated averaging). The general methodology for many of these works is to first propose a  technique for reducing the communication cost of SGD, and then demonstrate its effectiveness through experimental evaluations and/or prove its convergence under standard assumptions. Such convergence guarantees have been proven in \cite{QSGD, Qsparse-local-SGD, Stichetal, vqSGD} for various communication efficient training techniques.

In this paper, we take a different approach. We ask the following question: can we develop suitable statistical models for the stochastic gradients and use these models to inform the design of more efficient training techniques? This statistical perspective has been the focus of a recent research line in distributed estimation theory \cite{duchi,garg,braverman,raginski, yanjun,diakonikolas, archayaetal, allerton, Leightonarxiv}. Motivated by similar observations as above, these works focus on characterizing the impact of communication constraints on the estimation efficiency of common statistical models, e.g. Gaussian mean estimation and its sparse variants, and discrete distribution estimation. However, it is unclear how to leverage these results to inform new practical training methods with SGD; these works focus on a  parameter estimation framework rather the training problem, and their emphasis is on canonical statistical models such as Gaussian mean estimation, which may not accurately reflect the distribution of  the stochastic gradients.  

In this work, we connect these two research lines by casting each communication round of distributed SGD as a communication-constrained parameter estimation problem and propose a %several new
statistical model for the stochastic gradients. This model aims to capture the skewed and sparse distribution of the local gradients observed in experiments, which underlies the experimental success of existing sparsification methods such as top-$k$ and thresholding \cite{Aji,deepgrad}. These methods communicate only a few local gradients per node which have the largest magnitudes, and can reduce the communication cost to order $k\log d$, where $d$ is the number of parameters. %They have been observed to outperform (stochastic) quantization techniques in experiments \cite{Stichetal}, which require at least $\sqrt{d}$ bits.
 
By using the toolset of distributed estimation theory, we characterize the fundamental estimation efficiency for our proposed statistical model and the optimal communication and estimation schemes that achieve this performance. The study of this statistical model naturally leads to the idea of communicating a randomly chosen subset of a set of large magnitude gradients.  Instead of simply choosing the top-$k$ gradients, the optimal communication scheme chooses a random $k$-subset of the gradients with large magnitudes. Even though random-$k$ and top-$k$ sparsification have both been separately considered  in the literature, selecting a random subset of the top magnitude gradients, which corresponds to their concatenated application, is novel. We call this new sparsification strategy rTop-$k$. It can be observed from \cite{Qsparse-local-SGD, Stichetal, Stichnew} that  rTop-$k$ readily enjoys the same convergence guarantees as top-$k$ and random-$k$ since it is also a compression operator (Definition~3 of \cite{Qsparse-local-SGD}). However, in extensive experiments on both image and language domains with CIFAR-10, ImageNet, and Penn Treebank datasets we observe that rTop-$k$ significantly outperforms either top-$k$ or random-$k$ applied separately.

The contributions of our paper can be summarized as follows:
\begin{itemize}
\item We show that each communication round of distributed SGD can be cast as a communication-constrained statistical parameter estimation problem and that the study of such problems can inform the design of communication-efficient training schemes. To the best of our knowledge, our work is the first to connect distributed stastical estimation theory with communication efficient SGD methods.
\item We propose and analyze novel statistical estimation models that capture the skew and sparse distribution of the stochastic gradients. We develop new communication and estimation schemes for these models and prove their optimality via information theoretic lower bounds.
\item We introduce a new gradient sparsification scheme, rTop-$k$,  and show that it outperforms either of its two previously known constituent schemes in extensive experiments.
\end{itemize}

\section{Approach and Main Results} \label{Formulation}
\subsection{Distributed Statistical Parameter Estimation} We begin by formulating the distributed parameter estimation problem. Let
\begin{align*}
X_1, X_2, \cdots, X_n \overset{\text{i.i.d.}}{\sim} P_\theta, 
\end{align*}
where $\theta\in\Theta\subseteq \mathbb{R}^d$. We are interested in estimating the parameter $\theta$ from the samples $X_1,\dots, X_n$. Unlike the traditional statistical setting where samples $X_1,\cdots,X_n$ are available to the estimator as they are, we consider a distributed setting where each observation $X_i$ is available at a different node in a network and has to be communicated to a master node by using $k$ bits. In other words, each node has to encode its sample $X_i $ by a possibly randomized strategy $\Pi_i$ to a $k$-bit string $M_i$ independently of the other nodes and send it to the master processor. The goal of the master node is to produce an estimate  $\hat{\theta}$ of the underlying parameter $\theta$ from the $nk$-bit transcript $M=(M_1,\ldots,M_n)$ it receives from the nodes so as to minimize the  worst case squared $\ell^2$ risk:
\begin{equation}\label{minmax}
\inf_{(\{\Pi_i\}_{i=1}^n, \hat{\theta})}\sup_{\theta\in\Theta} \mathbb{E}_{\theta}\|\hat{\theta}(M)-\theta\|_2^2,
\end{equation}
where $\hat{\theta}(M)$ is an estimator of $\theta$ based on the transcript $M$. Note that the encoding strategies $\Pi_i$ for $i=1,\dots,n$ and the estimator $\hat{\theta}$ can be jointly designed to minimize the risk.

\subsection{Gradient Aggregation in Distributed Training} We next formulate the gradient aggregation problem in distributed training with SGD as a distributed parameter estimation problem. Let $g^t=\mathbb{E}[\nabla \ell (\omega_t;X)]$ denote the gradient of the population risk in $\mathbb{R}^d$ where $\ell$ denotes the loss function, $\omega_t\in\mathbb{R}^d$ denotes the weights of the network at the current iteration $t$, and the expectation is with respect to the true distribution of the samples $X$. Assume there are $n$ distributed nodes and each node $i$ for $i=1,\dots,n$ computes the (stochastic) gradient of the loss function with respect to  a small batch of i.i.d. samples $\{X_{j}^{(i)}\}_{j=1}^{B}$ available at this node,
$$
g_i^t=\frac{1}{B}\sum_{j=1}^B \nabla \ell (\omega_t;X_{j}^{(i)}).
$$
Note that when disjoint subsets of the dataset are assigned to different nodes, $\{X_{j}^{(i)}\}_{j=1}^{B}$ can be modeled as independent and identically distributed for different $i$. Hence, the local gradients $g_1^t, g_2^t,\dots, g_n^t\,\in\,\mathbb{R}^d$ are generated independently from the same (unknown) distribution  and have mean equal to $g^t$, the true gradient with respect to the population risk. \footnote{The same model can be extended to the case where nodes perform multiple steps of stochastic gradient descent on their local dataset such as in federated learning \cite{fedreview}, in which case $g_i^t$ is the resultant model update generated by node $i$.} To model the communication bottleneck, assume each node has $k$ bits to communicate its local gradient vector $g_i^t$ to the master node. The goal of the master node is to generate an estimate $\hat{g}^t$ of the true gradient $g^t$ under squared $\ell^2$ risk as per \eqref{minmax} from the  $nk$-bit transcript $M$ it obtains at the end of this communication round. The $g^t$ can be regarded as the parameter of the underlying distribution that we wish to estimate. In order to complete the model description,  we need to specify the statistical model $P_\theta$, with $\theta=g^t$, according to which the samples $g_1^t, g_2^t,\dots, g_n^t$ are generated and the set $\Theta\subseteq\mathbb{R}^d$ in which the parameter lies.

\begin{figure}
\centering
\begin{tikzpicture}
\node at (2, 1.4) {$g^t=\mathbb{E}[\nabla \ell (\omega;X)]$}; 
\draw [->] (1.8, 1.2) -- (0.2, 0.2); 
\draw [->] (1.9, 1.2) -- (1.1, 0.2); 
\draw [->] (2.1, 1.2) -- (2.9, 0.2); 
\draw [->] (2.2, 1.2) -- (3.8, 0.2); 
\draw (0,0) circle (0.2cm); \node [below] at (0,-0.2) {$\tilde{g}_1^t$}; 
\draw (1,0) circle (0.2cm); \node [below] at (1,-0.2) {$\tilde{g}_2^t$}; 
\node at (2,0) {$\cdots$}; 
\draw (3,0) circle (0.2cm); \node [below] at (3,-0.2) {$\tilde{g}_{n-1}^{t}$}; 
\draw (4,0) circle (0.2cm); \node [below] at (4,-0.2) {$\tilde{g}_{n}^{t}$}; 
\draw [->] (0, -0.7) -- (0, -1.4); \draw [->] (1, -0.7) -- (1, -1.4); 
\draw [->] (3, -0.7) -- (3, -1.4); \draw [->] (4, -0.7) -- (4, -1.4);
\draw (-0.3,-2.5) rectangle (4.3,-1.5);  % \node at (2,-3) {transcript $Y$}; 
\node at (2,-2) {centralized processor}; 
\draw [->] (4.3,-2) -- (5, -2); \node [right] at (5,-2) {$\hat{g}^t$}; 
%\draw [->] (2,-2.5) -- (2, -3); \node [below] at (2,-3) {$\hat{g}^t$}; 
\node [left] at (0, -1) {$k$ bits}; \node [right] at (4,-1) {$k$ bits}; 
\end{tikzpicture} 
\caption{Statistical estimation model for distributed training.}
\label{fig:fig1}
\end{figure}
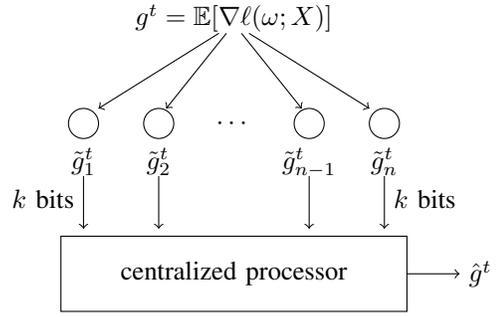 

\subsection{Statistical Models}

The distributed parameter estimation problem in \eqref{minmax} has been studied for specific classes of statistical models $P_\theta$ in the literature. One canonical model that has been of significant interest in the literature is the Gaussian location model $P_\theta=\mathcal{N}(\theta,\sigma^2I_d)$ with $\Theta=\mathbb{R}^d$ which is studied in \cite{duchi,garg,braverman,yanjun}. Adopting this Gaussian model for the gradient aggregation problem would correspond to modeling the local gradients $g_i^t$ computed at each node as i.i.d. observations of the true gradient vector $g_t$ under spherically-symmetric additive Gaussian noise. However, this leads to dense observation vectors $g_i^t$, while the distribution of the local gradients is often skewed and sparse like in experiments, i.e. often relatively few entries of $g_i^t$ have large magnitudes and many entries have magnitudes close to zero. 

A sparse variation of the Gaussian mean estimation problem has been studied in \cite{garg,braverman, yanjun}, where $\Theta= \{\theta\in \mathbb{R}^d:\|\theta\|_0\le s\le d\}$. Adopting this statistical model for the gradient aggregation problem would correspond to modeling the true gradient vector $g_t$ as sparse, with only $s$ non-zero components, while the local gradients $g_i^t$ computed at each node would still correspond to i.i.d. observations of the true gradient vector $g_t$ under additive Gaussian noise. Even though the true gradient vector $g_t$ is assumed to be sparse, qualitatively this model behaves similarly to the original Gaussian model. In particular the number of bits $k$ needed to achieve the centralized estimation performance (with no communication constraints) scales linearly with the ambient dimension $d$ as in the original Gaussian model, i.e. it is not impacted by the sparsity $s$. Similarly the optimal communication scheme independently quantizes each gradient component as in the case of the dense Gaussian model.

In this work, we instead propose to focus on the following sparse Bernoulli model. Suppose each node $i$ has a sample \begin{equation}\label{eq:model} X_i \sim \prod_{j=1}^d\text{Bern}(\theta_j)\end{equation} with $$\Theta = \left\{ \theta\in[0,1]^d \, : \, \sum_{j=1}^d \theta_j \leq s \right\} \; .$$ This model imposes a soft sparsity constraint on the parameter vector $\theta$; when $s\ll d$ not all entries of $\theta$ can be large, i.e. close to `1'. The observation vectors $X_i$ are highly skewed with entries either equal to `1' or `0'. Note also that the parameter vector $\theta$ corresponds to the mean of the observation vectors $X_i$. We adopt this model for the gradient aggregation problem with the understanding that $\theta=|g_t|$ and the `0' values from the Bernoulli random variables model entries of $g_i^t$ with small magnitudes, while the `1' values model entries with large magnitudes. Note that a large value for a given component of $\theta=|g_t|$ makes it more likely to observe `1', i.e. a large magnitude entry in the corresponding location of the stochastic gradient vectors $g_i^t$ for $i=1,\dots,n$. However, it does not exclude the possibility of observing a `0' i.e. a small magnitude entry, in the same location.

While this is a highly simplified model for real stochastic gradient vectors, it allows us to focus our attention on the impact of a highly skewed distribution for the magnitudes of the gradients. This model and our main results in the next section can be easily extended to more accurately reflect the distribution of the stochastic gradients, and doing so does not change the optimal encoding or estimation schemes. In particular, we could consider the following refinements:
\begin{itemize}
\item[(i)]{\bf Signed values:} The mean parameters $\theta_j$ could be in $[-1,1]$ subject to $\sum_{j=1}^d |\theta_j| \leq s$, so that  $\theta_j$'s can be positive or negative, with the $j$th component of the $X_i$ now being $\mathsf{Sign}(\theta_j)\mathsf{Bern}(|\theta_j|)$.
\item[(ii)]{\bf Scaling:} The Bernoulli random variables could be scaled by any $M>0$ with the goal of estimating the scaled Bernoulli parameters $M\theta$.
\item[(iii)]{\bf Continuous Perturbations:} We could add a vector of continuous independent zero-mean random variables $Z_i$ to each $X_i$,  where $X_i$ is distributed according to \eqref{eq:model} and each component of $Z_i$ is supported in $[-\frac{1}{2},\frac{1}{2}]$. As a result, each component of the resultant $Y_i=Z_i+X_i$ will now take continuous values.
\end{itemize}
With none of these refinements changing the fundamental encoding and estimation schemes, we find that the simple sparse Bernoulli model captures the issues that are central to our distributed estimation problem. Note that by including just the sign and scaling changes, the resulting model also makes a good representation of coarsely quantized stochastic gradients such as with ternary quantization.

\subsection{Theoretical Results and Discussion}

%This reflects the fact that the stochastic gradient vectors, which themselves are relatively sparse, may have large entries even in locations where the true gradient vector is small. The Bernoulli model captures the fact that the entries of the stochastic gradient vectors are also sparse like, with some entries having magnitudes much larger than others. This model may also be useful for analyzing stochastic gradient vectors after they have been thresholded, where many entries are set directly to zero, or after a coarse 1-bit quantization where the entries are set directly to zero or one.

We now turn to analyzing the distributed estimation problem under this model.

\begin{thm}[sparse Bernoulli upper bound]\label{thm:sparseBernoulliUB}
For the sparse Bernoulli model \eqref{eq:model} above,
$$\inf_{(\hat\theta,\Pi_i)}\sup_{\theta\in\Theta} \mathbb{E}_\theta\|\hat\theta-\theta\|_2^2 \leq C\frac{s^2\log d}{nk}$$
for $2\log d \leq k \leq s\log d$ and some constant $C$ that is independent of $n,k,d,s$.
\end{thm}

We prove this theorem in Section~\ref{sec:SparseBUB} by describing an explicit independent encoding scheme at each node and a centralized estimator that together achieve at most this error. The scheme is built on the idea of subsampling the non-zero entries in each sample vector and communicating the locations of the subsampled components. With signed values as in (i) above, the scheme is modified to communicate also the sign of each subsampled component. Since this requires only a single bit for each subsampled component it does change the scalings in Theorem~\ref{thm:sparseBernoulliUB}. Similarly, the constant scaling in (ii) also does not impact the scheme and the result of Theorem~\ref{thm:sparseBernoulliUB}. When continuous perturbations are present as in (iii), the scheme is modified to include a pre-processing step where $Y_i$ is quantized to `0' if $|Y_i|\leq 1/2$, and quantized to `1' if $|Y_i|>1/2$. This converts the observations to the original Bernoulli model, on which we apply the idea of random subsampling. This recovers the result in Theorem~\ref{thm:sparseBernoulliUB}. Note that this strategy is different than simply communicating the components of $Y_i$ with largest magnitude or the indexes of these large magnitude components.

We also prove the following lower bound which shows that our proposed scheme is order optimal up to logarithmic factors.

\begin{thm}[Sparse Bernoulli lower bound] \label{thm:sparseBernoulliLB}
For the sparse Bernoulli model \eqref{eq:model} above, if $nk \geq d\log\frac{d}{s}$ and $s\leq\frac{d}{2}$ then
\begin{equation}
\sup_{\theta\in\Theta} \mathbb{E}_\theta\|\hat\theta-\theta\|_2^2 \geq c\max\left\{\frac{s^2\log \frac{d}{s}}{nk},\frac{s}{n}\right\}\label{eq:sparseBerLB}
\end{equation}
for any estimator $\hat{\theta}(M)$, communication strategies $\Pi_i$, and some constant $c$ that is independent of $n,k,d,s$.
\end{thm} 
The proof is given in Section~\ref{sec:pf2}. It builds on a geometric characterization of Fisher information from quantized samples, a framework introduced in \cite{allerton}. The same lower bound applies trivially to extensions (i) and (ii) above, while the extension to (iii) follows simply from the data processing inequality for Fisher information. The two theorems together show that the number of bits $k$ (per node) needed to achieve the centralized performance under the sparse Bernoulli model is of the order of $s\log d$. (The centralized performance is given by the second term in the maximization on the right side of \eqref{eq:sparseBerLB}.) This model is interesting from a statistical estimation perspective; in contrast to the sparse Gaussian location model, this model suggests that the centralized performance can be achieved with much fewer than $d$ bits when the underlying vector is sparse with $s\ll d$.  

\section{$\text{r}$Top-$k$ Algorithm}
\label{sec:rtopk}
The main algorithmic idea that emerges from the theoretical analysis of the sparse Bernoulli model is to communicate a random subset of the large magnitude entries of the local gradient vectors. Motivated by this observation, we propose a new sparsification operator rTop-k for sparsifying the local gradient vectors, so as to reduce the overall communication cost of distributed SGD. This new sparsification operator can be viewed as a concatenation of top-$r$ and random-$k$ sparsification operators from the prior literature.

\begin{defn}[top-$r$ sparsification \cite{Aji,deepgrad,Stichetal,Qsparse-local-SGD}]
For a parameter $1\leq r\leq d$, the operator $\mathsf{top}_r: \mathbb{R}^d\rightarrow\mathbb{R}^d$ is defined for $\omega \in\mathbb{R}^d$ as
$$
\left(\mathsf{top}_r(\omega)\right)_i :=\begin{cases} (\omega)_i, & \; \text{if} \; i\in\{\pi(1),\ldots,\pi(r)\} \\ 0 & \; \text{otherwise ,} \end{cases}
$$
where $\pi$ is a permutation of $[d]:=\{1,\dots,d\}$ such that $|(\omega)_{\pi(i)}|\geq|(\omega)_{\pi(i+1)}|$ for $i=1,\dots, d-1$.
\end{defn}

\begin{defn}[random-$k$ sparsification \cite{federated2, Stichetal,Qsparse-local-SGD}]
For a parameter $1\leq k\leq d$, let $\mathcal{W}_k$ denote the set of all $k$ element subsets of $[d]$ and $W$ be uniformly at random chosen element of $\mathcal{W}_k$. The operator $\mathsf{random}_k: \mathbb{R}^d\rightarrow\mathbb{R}^d$ is defined for $\omega \in\mathbb{R}^d$ as
$$
\left(\mathsf{random}_k(\omega)\right)_i :=\begin{cases} (\omega)_{i}, & \; \text{if} \; i\in W \\ 0 & \; \text{otherwise.} \end{cases}
$$
\end{defn}

\begin{defn}[rTop-$k$ sparsification]\label{def:rTopk}
 Let $\pi$ be a permutation of $[d]$ such that $|(\omega)_{\pi(i)}|\geq|(\omega)_{\pi(i+1)}|$ for $i=1,\dots, d-1$. For two parameters $1\leq k\leq r\leq d$, let $\mathcal{U}_k$ denote the set of all $k$ element subsets of $\{\pi(1),\ldots,\pi(r)\}$ and $U$ be uniformly at random chosen element of $\mathcal{U}_k$.  The operator $\mathsf{rTop}_k: \mathbb{R}^d\rightarrow\mathbb{R}^d$ is defined for $\omega \in\mathbb{R}^d$ as
$$
\mathsf{rTop}_k (\omega)\,:=\begin{cases} (\omega)_{i}, & \; \text{if} \; i\in U \\ 0 & \; \text{otherwise.} \end{cases}.
$$
\end{defn}

%
%\begin{algorithm}[h]
%{\bf Hyperparameters:} number of components communicated $k$, subsampling ratio $r/k$ \\
%{\bf Inputs:} stochastic gradients $g_1^t,\ldots,g_n^t$ \\
%{\bf Output:} an estimate $\hat g^t$ of the true gradient $g^t$ \\
%\vspace{-.2in}
%\begin{algorithmic}
%\State {$\hat g^t \gets 0$}
%\For{$i=1,\ldots,n$} \State{$\hat g^t \gets \hat g^t + \mathsf{rTopEncode}(g_i^t)$} \EndFor
%\State {$\hat g^t \gets \left(\frac{r}{nk}\right)\cdot\hat g^t$} \\
%\State $\mathsf{rTopEncode}(g_i^t):$
%\State \hspace{.15in} $I \gets$ (indices of the $r$ largest mag. components of $g_i^t)$
%\State \hspace{.15in} $I \gets$ (a random subset of $k$ out of the $r$ elements of $I$)
%\State \hspace{.15in} $g \gets$ ($g_i^t$ with indices  not in $I$ set to 0)
%\State \hspace{.15in} {\bf return} g
%%\vspace{.1in}
%\end{algorithmic}
%\caption{rTop-$k$}
%\label{algorithm_rtopk}
%\end{algorithm}

The rTop-$k$ sparsification operator first chooses the $r$ entries with largest magnitudes of a vector $\omega\in\mathbb{R}^d$ and then communicates $k$ randomly chosen gradients among these $r$ entries. Note that while both top-$r$ and random-$k$ strategies have been studied extensively in the prior literature, their concatenated application as in rTop-$k$ has not been investigated in prior works. As we demonstrate in the next section, rTop-$k$ consistently outperforms either top-$r$ or random-$k$ applied alone in experiments. Intuitively, top-$r$ sparsification allows to focus only on the most significant entries of the gradient vector and random-$k$ sparsification allows to reduce the bias introduced by top-$r$ sparsification, hence concatenation combines the best of the two approaches. 

Note that using this algorithm, $k$ is the final number of components that must be communicated from each node to the centralized processor. Since the index for each component can be referred to with $\log d$ bits, and the value of each component can be encoded with a constant number of bits of precision, this is also up to logarithmic factors the number of bits needed for communication. For this reason we have used $k$ to refer to both the number of bits allowed in the distributed statistical estimation framework, and the number of gradient components that are communicated in distributed training.

%In Algorithm \ref{algorithm_rtopk_memory}, we show how the rTop-$k$ procedure can be used in distributed SGD with gradient accumulation \cite{deepgrad}. This is the algorithm we use for the simulations in the subsequent section. Distributed SGD with gradient accumulation and compression of the gradients via a ``$k$-contraction'' or ``compression operator'' has been analyzed in \cite{Stichetal} and \cite{Qsparse-local-SGD}, respectively. In these works they show for such operators (Definition 2.1 in \cite{Stichetal} and Definition 3 in \cite{Qsparse-local-SGD}), the SGD converges in the strongly convex \cite{Stichetal} and smooth but non-convex \cite{Qsparse-local-SGD} cases, and they also give upper bounds on the convergence rate. It is straightforward to show that rTop-$k$ satisfies these contraction/compression properties, and so the convergence of Algorithm \ref{algorithm_rtopk_memory} is guaranteed in these same cases by these same results. These results differ somewhat from the standard proofs of SGD convergence in that the gradient estimates can be unbiased, as is the case for rTop-$k$, however this does not affect the convergence of the algorithm.

The convergence of distributed SGD with sparsification operators top-$r$ and random-$k$ has been analyzed in the prior literature. Together with some form of error compensation, these methods have been shown to converge as fast as vanilla SGD in \cite{Qsparse-local-SGD, Stichetal, Stichnew}. Therefore, when applying rTop-$k$ sparsification to distributed SGD we also employ error compensation.  The algorithm we use for our simulations in the subsequent section is given in Algorithm \ref{algorithm_rtopk_memory}. In the next subsection, we prove that Algorithm \ref{algorithm_rtopk_memory} converges at the same rate as vanilla SGD for smooth but non-convex functions by building on the results of  \cite{Qsparse-local-SGD}. We note that this convergence result differs somewhat from the standard proofs of SGD convergence in that the gradient estimates can be unbiased, as is the case for rTop-$k$, however this does not affect the convergence of the algorithm when error compensation is employed.

\begin{algorithm}[!h]
{\bf Hyperparameters:} number of components communicated $k$, subsampling ratio $r/k$, learning rate $\eta$, minibatch size $B$ \\
%{\bf Inputs:} stochastic gradients $g_1^t,\ldots,g_n^t$ \\
{\bf Inputs:} local datasets $\mathcal{D}_i\quad i=1,\dots,n$ \\
%{\bf Output:} an estimate $\hat g^t$ of the true gradient $g^t$ \\
{\bf Output:} weights $\omega^T$ \\
\vspace{-.2in}
\begin{algorithmic}
\State{Broadcast randomly initialized weights, $\omega^0$, to all nodes.}
\State{$m_i^0 \gets 0$}
\For{$t=0,\ldots,T-1$} \\
\hspace{.15in} \underline{\textbf{On Nodes:}}
%\State{$\hat g^t \gets 0$}
\For{$i=1,\ldots,n$} 
\State{receive $\omega^{t}$ from the centralized processor.}
\State{$g_i^t \gets \frac{1}{B} \sum_{j=1}^B\nabla \ell (\omega_t;X_j^{(i)})$; $\{X_j^{(i)}\}_{j=1}^B$ is uniformly chosen from  $\mathcal{D}_i$}
\State{$g_i^t \gets g_i^t+ m_i^t$}
\State{$\hat g_i^t \gets  \mathsf{rTop}_k\,(g_i^t)$}
\State{$m_i^{t+1} \gets$ $g_i^t-\hat g_i^t$}
\State{send $\hat g_i^t$ to the centralized processor.} \EndFor \\
\hspace{.15in} \underline{\textbf{On Centralized Processor:}}
\State{receive $\hat g_i^t$'s from $n$ nodes.}
\State{$\hat g^t \gets \frac{1}{n}\sum_{i=1}^n \hat g_i^t$}  
\State {$\omega^{t+1} \gets \omega^{t} - \eta \cdot \hat g^t$} 
\State{broadcast $\omega^{t+1}$ to all nodes.}
\EndFor \\
\State $\mathsf{rTop}_k\,(g_i^t):$
%\State \hspace{.15in} $g_i^t \gets g_i^t+ m_i^t$
\State \hspace{.15in} $I \gets$ (indices of the $r$ largest mag. components of $g_i^t)$
\State \hspace{.15in} $I \gets$ (a random subset of $k$ out of the $r$ elements of $I$)
\State \hspace{.15in} $\hat g_i^t \gets$ ($g_i^t$ with indices  not in $I$ set to 0)
%\State \hspace{.15in} $m_i^{t+1} \gets$ $g_i^t-\hat g_i^t$
\State \hspace{.15in} {\bf return} $\hat g_i^t$ %, $m_i^{t+1}$
%\vspace{.1in}
\end{algorithmic}
\caption{SGD with rTop-$k$ and Error Compensation}
\label{algorithm_rtopk_memory}
\end{algorithm}

\subsection{Convergence of rTop-$k$}
In this section we demonstrate that a result from \cite{Qsparse-local-SGD} implies the convergence of Algorithm \ref{algorithm_rtopk_memory} to a fixed point under certain assumptions on the empirical risk. For nodes $i=1,\ldots,n$, let $\mathcal{D}_i$ be the local dataset at node $i$. Further define
$$f^{(i)}(\omega) = \mathbb{E}_{l\sim \mathcal{D}_i}[\ell(\omega;X_l)]$$
to be the local loss function at node $i$ where $\mathbb{E}_{l\sim\mathcal{D}_i}$ denotes expectation over a random sample $l$ chosen from $\mathcal{D}_i$. We consider minimization of the empirical risk
$$f(\omega) = \frac{1}{n}\sum_{i=1}^n f^{(i)}(\omega)$$
using Algorithm \ref{algorithm_rtopk_memory}.

We make the following assumptions:
\begin{itemize}
\item[(i)] {\bf Smoothness:} The local loss functions $f^{(i)}:\mathbb{R}^d\to\mathbb{R}$ are $L$-smooth for each $i=1,\ldots,n$ in the sense that
$$f^{(i)}(\omega) \leq f^{(i)}(\omega') + \langle \nabla f^{(i)}(\omega'),\omega-\omega'\rangle + \frac{L}{2}\|\omega-\omega'\|_2^2$$
for any $\omega,\omega'\in\mathbb{R}^d$.
\item[(ii)] {\bf Bounded second moment of gradients:} For each $i=1,\ldots,n$ and $\omega\in\mathbb{R}^d$, and some constant $0 \leq G < \infty$, we have $$\mathbb{E}_{l\sim\mathcal{D}_i}\|\nabla\ell(\omega;X_l)\|_2^2 \leq G^2 \; .$$
This also implies a bound on the variance $\sigma_i^2 = \mathbb{E}_{l\sim\mathcal{D}_i}\|\nabla\ell(\omega;X_l)-\nabla f^{(i)}(\omega)\|_2^2 \leq G^2 \; .$
\end{itemize}

%For $\omega\in\mathbb{R}^d$, we define $\mathsf{rTop}_k(\omega)$ to be a $d$-length vector obtained by first setting all but $r$ of the largest components of $\omega$ (by magnitude) to zero, and then uniformly at random subsampling $k$ out of the remaining $r$ components. 

Below, we show that the $\mathsf{rTop}_k$ operator in Definition~\ref{def:rTopk} is a ``compression operator,'' and that this implies certain convergence guarantees for the empirical-risk minimization. We note that the convergence of distributed SGD with compression operators has been established in the literature for  strongly convex \cite{Stichetal, Qsparse-local-SGD} and smooth but non-convex \cite{Qsparse-local-SGD} loss functions under the bounded second moment assumption for the gradients (Assumption (ii) above). Assumption (ii) is a standard assumption in \cite{SSS07,NJLS09,RRWN11,AHJ+18,Stichetal, Qsparse-local-SGD}. However, we note that the strongly convex assumption contradicts with Assumption (ii)  if the domain of the optimization problem is assumed to be all of $\mathbb{R}^d$ as is done in \cite{Stichetal}. Assumption (ii) has been subsequently removed in follow-up work \cite{Stichnew}. In this paper, we focus on the non-convex case where Assumption (ii) is standard and does not lead to a contradiction.

%In Algorithm \ref{algorithm_rtopk_memory}, we show how the rTop-$k$ procedure can be used in distributed SGD with gradient accumulation \cite{deepgrad}. This is the algorithm we use for the simulations in the subsequent section. Distributed SGD with gradient accumulation and compression of the gradients via a ``$k$-contraction'' or ``compression operator'' has been analyzed in \cite{Stichetal} and \cite{Qsparse-local-SGD}, respectively. In these works they show for such operators (Definition 2.1 in \cite{Stichetal} and Definition 3 in \cite{Qsparse-local-SGD}), the SGD converges in the strongly convex \cite{Stichetal} and smooth but non-convex \cite{Qsparse-local-SGD} cases, and they also give upper bounds on the convergence rate. It is straightforward to show that rTop-$k$ satisfies these contraction/compression properties, and so the convergence of Algorithm \ref{algorithm_rtopk_memory} is guaranteed in these same cases by these same results. These results differ somewhat from the standard proofs of SGD convergence in that the gradient estimates can be unbiased, as is the case for rTop-$k$, however this does not affect the convergence of the algorithm.

\begin{defn}[Compression operator]
A (potentially randomized) function $\mathsf{Comp}_k:\mathbb{R}^d\to\mathbb{R}^d$ is a compression operator if there exists a constant $\gamma\in(0,1]$ (potentially depending on $k$ and $d$) such that
$$\mathbb{E}_C\|\omega-\mathsf{Comp}_k(\omega)\|_2^2 \leq (1-\gamma)\|\omega\|_2^2$$
for each $\omega\in\mathbb{R}^d$, where the expectation is over the randomness in the operator.
\end{defn}

\begin{prop}
$\mathsf{rTop}_k$ is a compression operator with $\gamma=k/d$.
\end{prop}
\begin{proof}
Let $\omega_1,\ldots,\omega_d$ be an ordering of the $d$ components of $\omega$ such that $$|\omega_1|\geq|\omega_2|\geq\ldots\geq|\omega_d| \; .$$
Then
\begin{align*}
\mathbb{E}_C\|\omega-\mathsf{rTop}_k(\omega)\|_2^2 & = \mathbb{E}_C\left[\sum_{j=1}^d(\omega_j-(\mathsf{rTop}_k(\omega))_j)^2\right] \\
& = \left(1-\frac{k}{r}\right) \sum_{j=1}^r \omega_j^2 + \sum_{j=r+1}^d \omega_j^2 \\
& \leq \sum_{j=1}^d \omega_j^2 - \frac{k}{r}\frac{r}{d}\sum_{j=1}^d \omega_j^2 \\
& = \left(1-\frac{k}{d}\right)\|\omega\|_2^2 \; .
\end{align*}
\end{proof}
Using this proposition, we get the below convergence result that follows as a special case from Theorem 1 in \cite{Qsparse-local-SGD}. Let $f^* = \min_{\omega\in\mathbb{R}^d} f(\omega)$.
\begin{thm} \label{thm:convergence}
Under assumptions (i) and (ii) above, let the learning rate be set to $\eta = \frac{\widehat{C}}{\sqrt{T}}$ where $\widehat{C}$ is a constant that satisfies $\frac{\widehat{C}}{\sqrt{T}}\leq\frac{1}{2L}$. Let the $\omega^t$ be generated according to Algorithm \ref{algorithm_rtopk_memory}. Then
\begin{align*}\mathbb{E}\|\nabla f(z_T)\|_2^2  \leq & \left( \frac{\mathbb{E}[f(\omega^0)] - f^*}{\widehat{C}} + \widehat{C}L\left(\frac{G^2}{Bn}\right)\right)\frac{4}{\sqrt{T}} \\
&+ 8 \left(4\frac{\left(1-\frac{k^2}{d^2}\right)}{\frac{k^2}{d^2}}+1\right)\frac{\widehat{C}^2L^2G^2}{T}
\end{align*}
where $z_T$ is a random variable that is sampled uniformly from the $\omega^t$ each with probability $1/T$.
\end{thm}
\noindent Theorem \ref{thm:convergence} is exactly the same as Theorem 1 in \cite{Qsparse-local-SGD}, except with $H=1$, $\gamma = k/d$, and using $\sum_{i=1}^n \sigma_i^2 \leq nG^2$. Additionally, since we aggregate the parameter updates on every time step, the iterates $\omega^t$ are the same at every node and the random variable $z_T$ only varies over $T$ possible values instead of $nT$ possible values. Corollary 2 in \cite{Qsparse-local-SGD} clarifies the order of these terms, and we reproduce the corresponding result below for convenience.
\begin{cor}
Suppose $\mathbb{E}[f(\omega^0)] - f^* \leq J^2$ and pick $\widehat{C}^2 = \frac{Bn(\mathbb{E}[f(\omega^0)] - f^*)}{G^2L}$ which will satisfy $\frac{\widehat{C}}{\sqrt{T}}\leq\frac{1}{2L}$ if $T$ is sufficiently large. Then
$$\mathbb{E}\|\nabla f(z_T)\|_2^2 \leq O\left(\frac{JG}{\sqrt{BnT}}\right) + O\left(\frac{J^2Bnd^2}{k^2T}\right) \; .$$
\end{cor}

Theorem~\ref{thm:convergence} provides non-asymptotic guarantees for Algorithm~\ref{algorithm_rtopk_memory}, where we observe that sparsification does not affect the first order term. Here, we are required to decide the horizon $T$ before running the algorithm. Therefore, in order for the algorithm to converge to a fixed point, the learning rate needs to follow a piecewise schedule (i.e. the learning rate would have to be reduced once in a while throughout the training process), which is  what we do in our experiments in the next section. By similarly building on results of \cite{Qsparse-local-SGD} (Theorem~2), one can bound the convergence rate of Algorithm~\ref{algorithm_rtopk_memory} with a decaying learning rate as well as when nodes perform local iterations.

\section{Experiments}\label{sec:experiments}
\subsection{Experiment Settings}
We validate our approach over a wide range of experiments including both image and language domains with CIFAR-10, ImageNet, and Penn Treebank (PTB) datasets and using two training methods
where each node communicates the gradient vector after local training: (1) on one batch of the local data (distributed setting) and (2) for one epoch over the local data (i.e. one iteration over the local data adapting the federated setting in \cite{federated0, federated1}). Note that in the federated setting the number of communication rounds (per epoch) is much smaller than that in the distributed setting. In the federated setting one epoch always  corresponds to one communication round while in the distributed setting we perform multiple communications in each epoch (78 communication rounds/epoch for CIFAR-10 and 265 for PTB). In our CIFAR-10 and ImageNet experiments, the dataset is distributed to each user in an i.i.d. fashion, which corresponds to a homogeneous data distributions among the nodes. On the other hand, the experiment on PTB dataset assigns one chapter of the corpus to each user, resulting in a more heterogeneous scenario for the data distributions among the nodes. We have 5 nodes in our experiments. 

In Algorithm~\ref{algorithm_rtopk_memory}, we combine rTop-$k$ with distributed SGD and error compensation.  Here, we add the error, i.e. the gradients that are not communicated in the current iteration, to the gradients computed in the next iteration so that all important gradients are communicated eventually as in \cite{deepgrad}. %We set the hyperparameter of subsampling ratio as the number of nodes in our experiments motivated by the observation that in expectation a parameter would be updated by one node if it appears in the set of largest magnitude updates at each node. The overall compression ratio is determined by $k$. 
We fix the hyperparameter $r$ such that ${k}/{r} = {1}/{n}$ in all of our experiments. This choice is motivated by the following intuition: if the gradients corresponding to a given parameter are among the top $k$ largest magnitude gradients at all $n$ nodes,  this choice ensures that in expectation this parameter will be updated by one node, since each node chooses to update this parameter with probability ${k}/{r} = {1}/{n}$. As an example, a compression ratio of $99.9\%$ requires only $k = 0.1\%$ of the entries of each gradient vector to be communicated by the corresponding node. For $n=5$ nodes this is achieved by taking $\frac{r}{d}=0.5\%$ of the entries of each gradient vector ($d$ is the number of parameters in the model), which have the largest magnitudes, and then communicating a random subset of $\frac{k}{r}=\frac{1}{n}=20\%$ of them. The central node calculates the global update vector by averaging the updates it receives for each component. 

We compare our results with the baseline setting where there is no compression, the setting proposed in \cite{deepgrad}, which uses the gradient top-$k$ selection method, and the random-$k$ sparsification strategy where the gradients to be communicated by each node are chosen uniformly at random \cite{federated2}. We employ the warm-up strategy and exponentially increase the sparsity in the warm-up period as in \cite{deepgrad}. We further employ the local gradient accumulation strategy in \cite{deepgrad}, which provides substantial improvement in the performance for all sparsification methods. We do not compare our strategy to stochastic quantization  techniques such as \cite{QSGD} as our main goal is to compare the performance of different sparsification methods. Also, \cite{Stichetal} observes that top-$k$ selection outperforms stochastic quantization  in experiments and theoretically it can reduce the communication cost to at most order $\sqrt{d}$ bits, while sparsification methods can achieve order $k\log d$, where $d$ is the number of parameters. %

\subsection{Image Domain}
For image classification tasks, we trained ResNet-18 \cite{resnet} on CIFAR-10 \cite{cifar_dataset} and ResNet-34 \cite{resnet} on ImageNet \cite{imagenet_cvpr09}. CIFAR-10 dataset consists of 50,000 training images, 5000 for each class, and 10000 test images, 1000 for each class. ImageNet dataset contains over 1 million training images and 50,000 validation images in 1000 classes. In all experiments, we use momentum SGD as the optimizer and cross entropy as the loss function. We split the data into batches of size 128 for CIFAR-10 and 32 for ImageNet experiments. We set the warm-up period as 5 epochs.

In our first two experiments, we train ResNet-18 on CIFAR-10 in both distributed and federated settings and show the results in Figure~\ref{fig:CIFAR10_distributed} and Figure~\ref{fig:CIFAR10_federated} respectively. Figure~\ref{fig:ImageNet} shows the results of training ResNet-34 on ImageNet dataset. Figures depict the performance of baseline with no compression and our rTop-$k$ strategy with compression ratios $99\%$ and $99.9\%$. To compare our method with \cite{deepgrad}, we present the performance of top-$k$ strategy with compression ratios $99\%$ and $99.9\%$. Final accuracies and compression ratios are summarized in Tables~\ref{tab:cifar10_dist}, \ref{tab:cifar10_fed}, and \ref{tab:ImageNet}. %\berivan{We do not include the initial points in the figures as all settings start from the same initial point.}

We note that in all experiments rTop-$k$ strategy has substantially better performance compared to top-$k$ and random-$k$ strategies under the same compression ratio. In the federated settings in Table~\ref{tab:cifar10_fed} and \ref{tab:ImageNet}, we observe that the accuracy of rTop-$k$ with $99.9\%$ compression ratio is better than the performance of top-$k$ $99\%$ compression ratio, corresponding to more than an order of magnitude improvement in compression ratio. Another interesting observation is that under $99\%$ compression ratio rTop-$k$ outperforms the baseline.

     \begin{figure*}[!h]
      \begin{minipage}{0.5\textwidth}
       \centering
        \includegraphics[width=\textwidth]{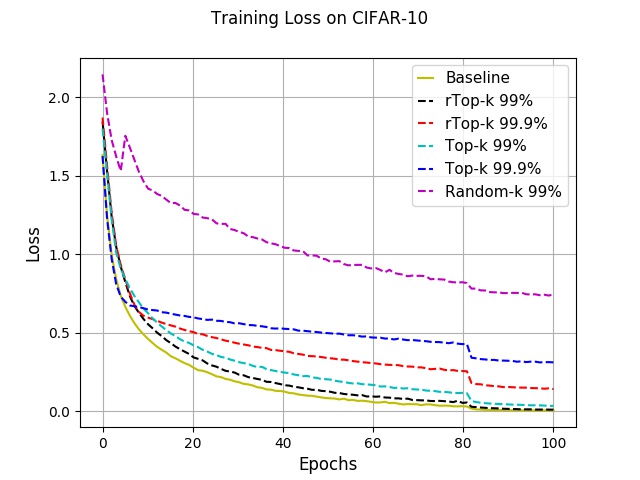}
      \end{minipage}\hfill
      \begin{minipage}{0.5\textwidth}
        \centering
        \includegraphics[width=\textwidth]{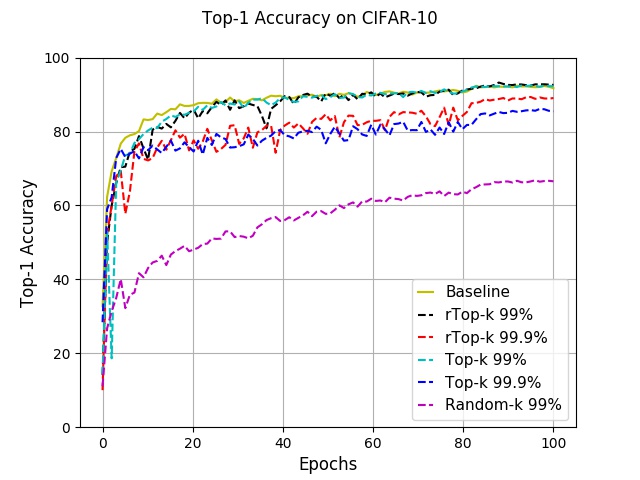}
      \end{minipage}\hfill
      \centering
      \caption{Training Loss and Top-1 Test Accuracy on CIFAR-10 (distributed setting).}\label{fig:CIFAR10_distributed}
   \end{figure*}

   \begin{figure*}[!h]
      \begin{minipage}{0.5\textwidth}
        \centering
        \includegraphics[width=\textwidth]{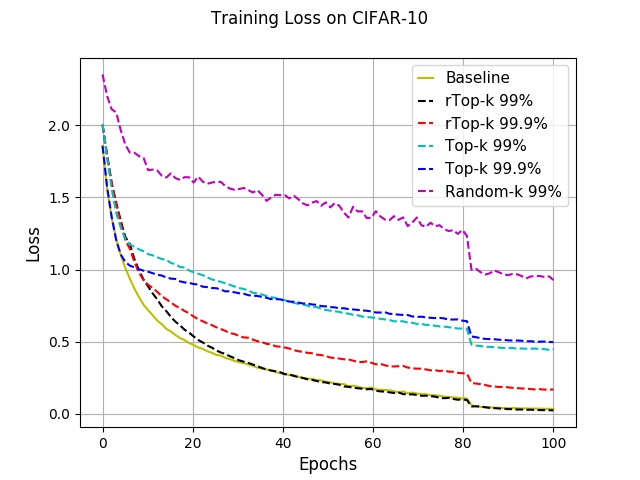}
      \end{minipage}\hfill
      \begin{minipage}{0.5\textwidth}
        \centering
        \includegraphics[width=\textwidth]{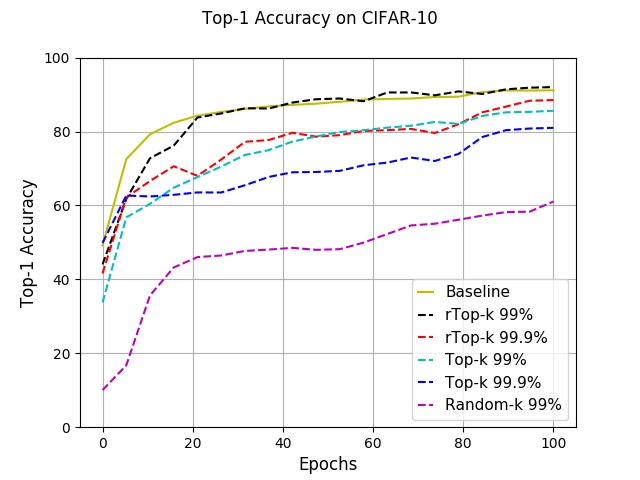}
      \end{minipage}\hfill
      \centering
      \caption{Training Loss and Top-1 Test Accuracy on CIFAR-10 (federated setting).}\label{fig:CIFAR10_federated}
   \end{figure*}
 
    \begin{figure*}[!h]
      \begin{minipage}{0.5\textwidth}
        \centering
        \includegraphics[width=\textwidth]{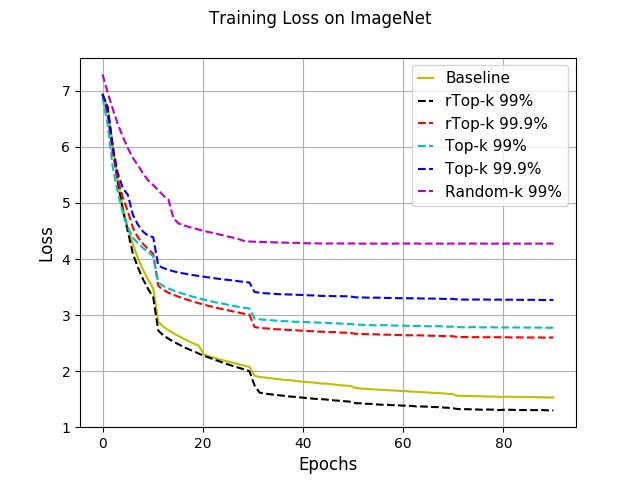}
      \end{minipage}\hfill
      \begin{minipage}{0.5\textwidth}
        \centering
        \includegraphics[width=\textwidth]{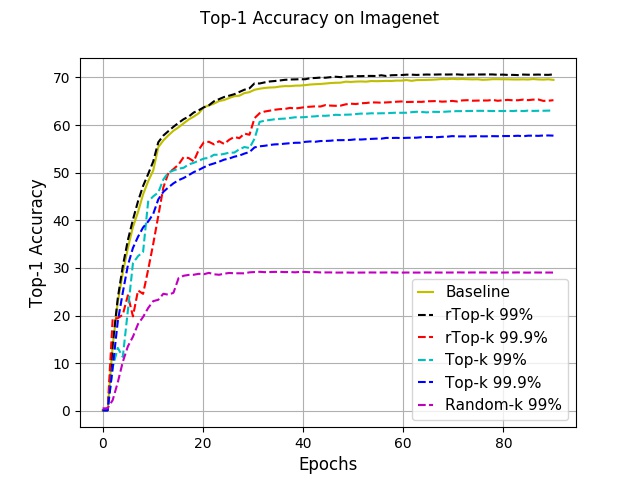}
      \end{minipage}\hfill
      \centering
      \caption{Training Loss and Top-1 Test Accuracy on ImageNet (federated setting).}\label{fig:ImageNet}
   \end{figure*}

\begin{table}[!h]
\centering
\caption{ResNet-18 trained on CIFAR-10 (distributed setting).}   
\label{tab:cifar10_dist}
\arrayrulecolor{black}
\begin{tabular}{!{\color{black}\vrule}l!{\color{black}\vrule}l!{\color{black}\vrule}l!{\color{black}\vrule}l!{\color{black}\vrule}l!{\color{black}\vrule}} 
\hline
Method                   & Top-1 Accuracy  & Compression  \\ 
\hline
Baseline                             & 92.40\%                           & -            \\ 
\hline
rTop-$k$                           & \textbf{93.25\%}                          & 99\%         \\ 
\hline
rTop-$k$                        & 89.34\%                          & 99.9\%         \\ 
\hline
Top-$k$              & \textbf{92.46\%}                          & 99\%      \\ 
\hline
Top-$k$             & 86.12\%                             & 99.9\%      \\
\hline
Random-k             & 66.81\%                             & 99\%      \\
\hline
\end{tabular}
\arrayrulecolor{black}
\end{table}

\begin{table}[!h]
\centering
\caption{ResNet-18 trained on CIFAR-10 (federated setting).}
\label{tab:cifar10_fed}
\arrayrulecolor{black}
\begin{tabular}{!{\color{black}\vrule}l!{\color{black}\vrule}l!{\color{black}\vrule}l!{\color{black}\vrule}l!{\color{black}\vrule}l!{\color{black}\vrule}} 
\hline
Method                   & Top-1 Accuracy & Compression  \\ 
\hline
Baseline                             & 91.16\%                           & -            \\ 
\hline
rTop-$k$                           & \textbf{92.02\%}                          & 99\%         \\ 
\hline
rTop-$k$                         & 88.51\%                          & 99.9\%         \\ 
\hline
Top-$k$              & 85.62\%                           & 99\%      \\ 
\hline
Top-$k$            & 81.00\%                             & 99.9\%      \\
\hline
Random-k             & 61.07\%                             & 99\%      \\
\hline
\end{tabular}
\arrayrulecolor{black}
\end{table}

\begin{table}[!h]
\centering
\caption{ResNet-34 trained on ImageNet (federated setting).}   
\label{tab:ImageNet}
\arrayrulecolor{black}
\begin{tabular}{!{\color{black}\vrule}l!{\color{black}\vrule}l!{\color{black}\vrule}l!{\color{black}\vrule}l!{\color{black}\vrule}l!{\color{black}\vrule}} 
\hline
Method                   & Top-1 Accuracy  & Compression  \\ 
\hline
Baseline                             & 69.70\%                           & -            \\ 
\hline
rTop-$k$                          & \textbf{70.63\%}                          & 99\%         \\ 
\hline
rTop-$k$                         & 65.37\%                          & 99.9\%         \\ 
\hline
Top-$k$              & 63.06\%                          & 99\%      \\ 
\hline
Top-$k$             & 57.80\%                             & 99.9\%      \\
\hline
Random-$k$             & 29.19\%                             & 99\%      \\
\hline
\end{tabular}
\arrayrulecolor{black}
\end{table}
\subsection{Language Domain}
We use the Penn Treebank corpus (PTB) dataset, which consists of 923,000 training,
73,000 validation and 82,000 test words \cite{penn_treebank}. We train the 2-layer LSTM language model architecture with 1500 hidden units per layer \cite{press-wolf-2017} and tie the input and the output embeddings \cite{Inan2016}. We use the same train/validation/test set split and vocabulary as \cite{Mikolov10}. We use vanilla SGD with gradient clipping. We use the same hyperparameters (i.e. weight initialization, learning rate schedule, batch size) as in \cite{Zaremba14}. We set the warm-up period to 5 epochs.

In our first experiment, we study the distributed setting where each node communicates the gradient vector after the forward-backward pass on a single batch of local data. Figure \ref{fig:language_distributed} shows the perplexity and training loss of the trained language model in this setting. 

\begin{figure*}[!h]
\begin{minipage}{0.5\textwidth}
\centering
\includegraphics[width=\textwidth]{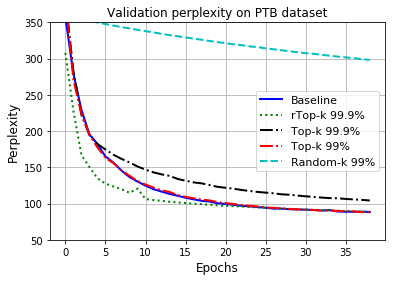}
\end{minipage}\hfill
\begin{minipage}{0.5\textwidth}
\centering
\includegraphics[width=\textwidth]{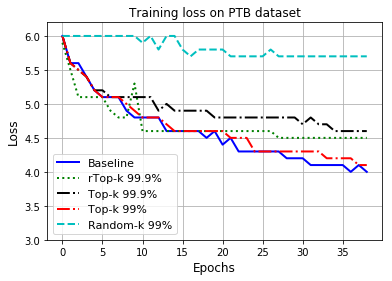}
\end{minipage}\hfill
\centering
\caption{Perplexity and training loss of LSTM language model on PTB dataset (distributed setting).}
\label{fig:language_distributed}
\end{figure*}

\begin{table}[!h]
\centering
\caption{Training results of language modeling on PTB dataset (distributed setting).}
\label{tab:language_distributed}
\arrayrulecolor{black}
\begin{tabular}{!{\color{black}\vrule}l!{\color{black}\vrule}l!{\color{black}\vrule}l!{\color{black}\vrule}l!{\color{black}\vrule}} 
\hline
Method                   &  Perplexity & Compression  \\ 
\hline
Baseline                   & 84.63         & -            \\ 
\hline
rTop-$k$  & \textbf{82.49}         & 99.9\%      \\ 
\hline
Top-$k$ &  {91.84}         & 99.9\%      \\
\hline
Top-$k$ &  \textbf{84.31}      & 99\%      \\
\hline
Random-$k$ &  281.61     & 99\%      \\
\hline
\end{tabular}
\arrayrulecolor{black}
\end{table}

In the second experiment, we study the federated setting where each node communicates the gradient vector after local training for one epoch over the local data. Figure \ref{fig:language_federated} shows the perplexity and training loss of the trained language model in this setting. We did not use the local gradient accumulation strategy in this setting since we did not observe an improvement in the performance. 
 
 \begin{figure*}[!h]
\begin{minipage}{0.5\textwidth}
\centering
\includegraphics[width=\textwidth]{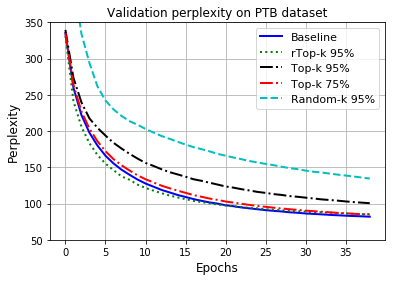}
\end{minipage}\hfill
\begin{minipage}{0.5\textwidth}
\centering
\includegraphics[width=\textwidth]{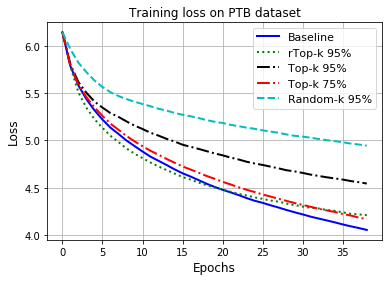}
\end{minipage}\hfill
\centering
\caption{Perplexity and training loss of LSTM language model on PTB dataset (federated setting).}
\label{fig:language_federated}
\end{figure*}

\begin{table}[!h]
\centering
\caption{Training results of language modeling on PTB dataset (federated setting).}
\label{tab:language_federated}
\arrayrulecolor{black}
\begin{tabular}{!{\color{black}\vrule}l!{\color{black}\vrule}l!{\color{black}\vrule}l!{\color{black}\vrule}l!{\color{black}\vrule}} 
\hline
Method                   &  Perplexity & Compression  \\ 
\hline
Baseline                   & 82.14         & -            \\ 
\hline
rTop-$k$ & \textbf{82.02}         & 95\%      \\ 
\hline
Top-$k$ &  {97.05}         & 95\%      \\
\hline
Top-$k$ &  \textbf{81.97}      & 75\%      \\
\hline
Random-$k$ &  130.91     & 95\%      \\
\hline
\end{tabular}
\arrayrulecolor{black}
\end{table}

In our first experiment, we observe from Figure \ref{fig:language_distributed} that the validation perplexity of our rTop-$k$ strategy matches the baseline with 99.9\% compression ratio. The top-$k$ strategy achieves the same performance with 99\% compression ratio. We also note that the random-$k$ strategy has substantially worse performance in this experiment. A similar set of results is obtained in the second experiment with slightly less aggressive compression levels compared to the first experiment as shown in Figure \ref{fig:language_federated}. %\berivan{We do not include the initial points in the figures as all settings start from the same initial point.} %This is perhaps because there is less gradient communications from the nodes as they locally train for one epoch before communicating the gradients, therefore, communicating more gradients might be more critical in this scenario. 
In all cases, the rTop-$k$ strategy has substantially better performance compared with the other methods at the same compression ratio. In both experiments, we observe from training losses and the corresponding test results (Figure \ref{fig:language_distributed} and Table \ref{tab:language_distributed} for the first experiment and Figure \ref{fig:language_federated} and Table \ref{tab:language_federated} for the second experiment) that the rTop-$k$ strategy obtains a better perplexity compared to baseline while its training loss is slightly larger. This suggests that training with rTop-$k$ strategy results in better generalization.

\section{Proof of Theorem~\ref{thm:sparseBernoulliUB}}\label{sec:SparseBUB}

To prove the desired upper bound, we describe an explicit encoding scheme $\Pi_i:\{0,1\}^d \to \{0,1\}^k$ and estimator $\hat\theta$ that is a function of the $M_i = \Pi_i(X_i)$ that achieves at most this error. The encoding functions $\Pi_i$ are defined as follows:
\begin{itemize}
\item[(i)]The first $\log d$ bits of $M_i = \Pi_i(X_i)$ encode the number of nonzero entries in $X_i$, i.e., $\|X_i\|_1$.
\item[(ii)] We form a codebook for the remaining $k-\log d$ bits such that each element of $\{0,1\}^d$ with at most $k'$ ones maps to a unique $k-\log d$ bit string. Note that we can set $k' \geq \frac{k-\log d}{\log d}$.
\item[(iii)] Take $X_i$, and if $\|X_i\|_1 > k'$, we form $\tilde X_i$ by uniformly at random keeping only $k'$ out of the original $\|X_i\|_1$ ones. If $\|X_i\|_1 \leq k'$, then $\tilde X_i = X_i$. We encode $\tilde X_i$ using the codebook from step (ii).
\end{itemize}
With this encoding scheme, the estimator $\hat\theta$ has access to both $\tilde X_i$ and $\|X_i\|_1$ for $i=1,\ldots,n$. For convenience, define the subsampling fraction $S_i$ by
\begin{align*} S_i = \begin{cases} \frac{k'}{\|X_i\|_1} & \; \text{if} \; \|X_i\|_1 > k' \\ 1 & \; \text{otherwise .} \end{cases} \end{align*}
Our estimator $\hat\theta$ is now defined by $\hat\theta = \frac{1}{n}\sum_{i=1}^n \frac{\tilde X_i}{S_i} \; .$
This estimator is unbiased in that
\begin{align*}
\mathbb{E} [\hat\theta] & = \frac{1}{n}\sum_{i=1}^n \mathbb{E}\left[\frac{\tilde X_i}{S_i}\right] \\ & = \frac{1}{n}\sum_{i=1}^n \mathbb{E}\left[\mathbb{E}\left[\frac{\tilde X_i}{S_i}\bigg| S_i \right]\right] \\
& = \frac{1}{n}\sum_{i=1}^n \mathbb{E}\left[\mathbb{E}\left[X_i | S_i \right]\right] \\ & = \frac{1}{n}\sum_{i=1}^n \mathbb{E}\left[ X_i \right] \\ & = \theta \; .
\end{align*}
Finally, we compute the error:
\begin{align*}
\mathbb{E}[(\hat\theta_j - \theta_j)^2]  = & \mathbb{E}[\hat\theta_j^2] - \theta_j^2 \\
 = & \frac{1}{n^2} \sum_{i=1}^n \mathbb{E}\left[\frac{\tilde X_{i,j}^2}{S_i^2}\right]  \\
& + \frac{1}{n^2}\sum_{i\neq k} \mathbb{E}\left[\frac{\tilde X_{i,j}\tilde X_{k,j}}{S_iS_k}\right]  -\theta_j^2 \\
 = & \frac{1}{n^2} \sum_{i=1}^n \mathbb{E}\left[\mathbb{E}\left[\frac{\tilde X_{i,j}}{S_i^2}\bigg| S_i \right]\right]  \\
& + \frac{1}{n^2}\sum_{i\neq k}\mathbb{E}\left[\frac{\tilde X_{i,j}}{S_i}\right]\mathbb{E}\left[\frac{\tilde X_{k,j}}{S_k}\right] -\theta_j^2 \\
 = & \frac{1}{n^2} \sum_{i=1}^n \mathbb{E}\left[\mathbb{E}\left[\frac{ X_{i,j}}{S_i}\bigg| S_i \right]\right] \\
& + \frac{n(n-1)}{n^2}\theta_j^2 -\theta_j^2 \\
 \leq & \frac{1}{n^2} \sum_{i=1}^n \mathbb{E}\left[\mathbb{E}\left[\frac{ X_{i,j}}{S_i}\bigg| S_i \right]\right] \; ,
\end{align*}
and then summing over each component $j$,
\begin{align}
\mathbb{E}\|\hat\theta-\theta\|_2^2 & \leq \frac{1}{n^2} \sum_{i=1}^n \mathbb{E}\left[\mathbb{E}\left[\frac{ \|X_i\|_1}{S_i}\bigg| S_i \right]\right] \nonumber\\
& \leq \frac{1}{n^2} \sum_{i=1}^n \left(\mathbb{E}\left[\frac{\|X_i\|_1^2}{k'}\right]+k'\right) \label{eq:achiev1}\\
& \leq \frac{1}{n^2} \sum_{i=1}^n \left(\mathbb{E}\left[\frac{\|X_i\|_1^2}{k'}\right]+C_1s\right) \label{eq:achiev2}\\
& \leq C_2\frac{s^2\log d}{nk} \label{eq:achiev3} \; .
\end{align}
In the displays above, \eqref{eq:achiev1} follows by separating out the cases $S_i =1$ and $S_i < 1$. In any case where $S_i<1$, the ratio inside the conditional expectation is exactly $k'$, and in any case where $S_i = 1$, the fraction $\frac{\|X_i\|_1 }{S_i} =\|X_i\|_1 \leq k'$. The step in \eqref{eq:achiev2} follows because $k' \leq C_1\frac{k}{\log d}$ and we are assuming $\frac{k}{\log d} \leq s$, and \eqref{eq:achiev3} uses the second moment formula for a Poisson binomial distribution.

%\appendices
%\section{Supplementary Material}
\section{Proof of Theorem~\ref{thm:sparseBernoulliLB}} \label{sec:pf2}
To prove Theorem~\ref{thm:sparseBernoulliLB}, we use part of Theorem 3 from \cite{isit} (see also \cite{Leightonarxiv}), reproduced below, which is proved using a Fisher information argument. Recall that in the context of Fisher information, the score function vector is the gradient of the log-likelihood, i.e $S_\theta(x) = \nabla_\theta\log f(x|\theta)$. Recall also that the $\Psi_2$-Orlicz norm of a random variable $X$ is defined as
$$\|X\|_{\Psi_2} = \inf\{K \in (0,\infty) \; | \; \mathbb{E}[\Psi_2(|X|/K)]\leq1\}$$
where
$\Psi_2(x) = \exp(x^2) - 1 \; ,$
and that a random variable with finite $\Psi_2$-Orlicz norm is sub-Gaussian \cite{versh}.

\begin{thm}[Barnes, Han, \"Ozg\"ur 2019] \label{thm:lower_bound}
Suppose $\Theta = [-B,B]^d$. For any estimator $\hat{\theta}(M)$ and communication strategies $\Pi_i$, if $S_\theta(X)$ satisfies $\|\langle u,S_\theta(X)\rangle\|_{\Psi_2}\leq N$ for any unit vector $u\in\mathbb{R}^d$, then
\begin{align*}
\sup_{\theta\in\Theta} \mathbb{E} & [\|\hat\theta - \theta\|^2] \nonumber \geq \frac{d^2}{CN^2n + \frac{d\pi^2}{B^2}} \; .
\end{align*}
\end{thm}
%\begin{cor}[sparse Bernoulli lower bound]
%For the sparse Bernoulli model above, if $nk \geq d\log\frac{d}{s}$ then
%$$\sup_{\theta\in\Theta} \mathbb{E}_\theta\|\hat\theta-\theta\|_2^2 \geq c\max\left\{\frac{s^2\log \frac{d}{s}}{nk},\frac{s}{n}\right\}$$
%for any estimator $\hat{\theta}(Y)$, communication protocol  $\Pi \in \Pi_{\mathsf{BB}}$, and some constant $c$ that is independent of $n,k,d,s$.
%\end{cor}
\begin{proof}[Proof of Theorem~\ref{thm:sparseBernoulliLB}]
The $\frac{s}{n}$ lower bound comes from considering the centralized case where there is no communication constraint. We will focus on the other bound, and will restrict our attention to a subset $\Theta'\subset\Theta$ and then use the fact that
$$\sup_{\theta\in\Theta} \mathbb{E}_\theta\|\hat\theta-\theta\|_2^2 \geq \sup_{\theta\in\Theta'} \mathbb{E}_\theta\|\hat\theta-\theta\|_2^2 \; .$$
In particular, let $\Theta' = \left[\frac{s}{2d},\frac{s}{d}\right]^d$. The distribution $f(x|\theta)$ is the product of Bernoulli distributions with parameters $\theta_i$, so the score function for each component $\theta_i$ is
$$S_{\theta_i}(x_i) = \frac{\partial}{\partial\theta_i}\log f(x_i|\theta_i) = \begin{cases} \frac{1}{\theta_i} & \; \text{, if } \; x_i=1 \\ \frac{-1}{1-\theta_i} & \; \text{, if } \; x_i =0 \; .\end{cases}$$
If we set
\begin{align*}
N = \max\left\{\frac{1}{\theta_i\sqrt{\log\frac{1}{\theta_i}}} \; , \; \frac{1}{(1-\theta_i)\sqrt{\log\frac{1}{(1-\theta_i)}}}  \right\} \; ,
\end{align*}
then in can be checked that
\begin{align*}
\mathbb{E}\left[e^{\left(\frac{S_{\theta_i}(X_i)}{N}\right)^2}\right] & = \theta_i e^{\left(\frac{1}{\theta_iN}\right)^2} + (1-\theta_i) e^{\left(\frac{1}{(1-\theta_i)N}\right)^2} \\
& \leq 2
\end{align*}
and thus $\|S_{\theta_i}(X_i)\|_{\Psi_2} \leq N$.
By taking sums of scaled independent sub-Gaussian random variables \cite{versh} we get
\begin{align*}
\|\langle u,S_\theta(X) \rangle\|_{\Psi_2}  \leq & c_1 \max\bigg\{\frac{1}{\theta_i\sqrt{\log\frac{1}{\theta_i}}} \; , \\
& \quad \frac{1}{(1-\theta_i)\sqrt{\log\frac{1}{(1-\theta_i)}}}  \bigg\} \\
& \leq c_1 \frac{1}{\theta_i\sqrt{\log\frac{1}{\theta_i}}} \leq \frac{c_2d}{s\sqrt{\log\frac{d}{s}}}
\end{align*}
where in the last line we have used the fact that $x^2\log x \geq (1-x)^2\log(1-x)$ for $0\leq x \leq \frac{1}{2}$ and $s\leq\frac{d}{2}$ so that $\frac{s}{d}\leq \frac{1}{2}$. Then by Theorem \ref{thm:lower_bound} above,
\begin{align*}
\sup_{\theta\in\Theta'} \mathbb{E}_\theta\|\hat\theta-\theta\|_2^2 & \geq \frac{d^2}{c_3nk\frac{d^2}{s^2\log\frac{d}{s}}+c_4\frac{d^3}{s^2}} \\
& \geq c_5\frac{s^2\log\frac{d}{s}}{nk}
\end{align*}
where the last inequality uses $nk\geq d\log\frac{d}{s}$.
\end{proof}

%\section*{Acknowledgements}
%This work was supported in part by NSF award CCF-1704624 and by a Google faculty research award.

\bibliographystyle{IEEEtran}
\bibliography{di.bib}

%\begin{IEEEbiography}{Leighton Pate Barnes}
%\end{IEEEbiography}
%
%\begin{IEEEbiography}{Huseyin Inan}
%\end{IEEEbiography}
%
%\begin{IEEEbiography}{Berivan Isik}
%\end{IEEEbiography}
%
%\begin{IEEEbiography}{Ayfer \"{O}zg\"{u}r}
%\end{IEEEbiography}


\end{document}